\documentclass{article}
\usepackage{ijcai19}

\usepackage{times}
\usepackage{soul}
\usepackage{url}
\usepackage[hidelinks]{hyperref}
\usepackage[utf8]{inputenc}
\usepackage[small]{caption}
\usepackage{graphicx}
\usepackage{amsmath}
\usepackage{booktabs}
\usepackage{algorithm2e}
\usepackage{algorithmic}
\usepackage{paralist}

\newtheorem{theorem}{Theorem}
\newtheorem{lemma}[theorem]{Lemma}
\newtheorem{definition}[theorem]{Definition}
\newtheorem{example}[theorem]{Example}
\newtheorem{proposition}[theorem]{Proposition}

\def\FullBox{\hbox{\vrule width 8pt height 8pt depth 0pt}}

\newcommand{\qed}{\;\;\;\FullBox}

\newenvironment{proof}{\noindent{\bf Proof:~~}}{$\qed$}

\newcommand{\calC}{\mathcal{C}}

\newcommand{\calP}{\mathcal{P}}

\begin{document}
\title{Query-driven PAC-Learning for Reasoning}
\author{Brendan Juba\\
Washington University in St.\ Louis\\
{\tt bjuba@wustl.edu}}

\maketitle
\begin{abstract}
We consider the problem of learning rules from a data set that support a proof
of a given query, under Valiant's PAC-Semantics. We show how any backward
proof search algorithm that is sufficiently oblivious to the contents of its
knowledge base can be modified to learn such rules while it searches for a proof
using those rules. We note that this gives such algorithms for standard logics
such as chaining and resolution.
\end{abstract}

\section{Introduction}

Machine learning, coupled with plentiful data, promises an approach to
the problem of constructing the large knowledge bases needed for AI.
Whereas traditional knowledge engineering by hand, as exemplified by the CYC 
project~\cite{lenat95}, proved difficult to scale, machine
learning holds the promise of producing a large and consistently interpreted
knowledge base. Of course, any kind of inductive learning faces the danger of
incorrect generalization, and thus such knowledge must use a semantics that is
weaker than classical logic. Valiant~\shortcite{valiant00} proposed 
{\em PAC-Semantics} as a semantics for classical logics that is liberal enough 
to tolerate the imperfect rules produced by models of machine 
learning~\cite{valiant84}. Subsequently, Valiant~\shortcite{valiant06} also
demonstrated that a large knowledge base can be soundly learned from a 
reasonable size data set. Michael and Valiant~\shortcite{mv08} demonstrated
the use of such knowledge bases on a sentence completion task, and
a system using this approach~\cite{im16} was tied for top performance in the
first Winograd Schema Challenge competition~\cite{dmo17}.

Although Valiant~\shortcite{valiant06} showed that there is no {\em statistical}
barrier to learning a large knowledge base, {\em computational} issues from
representing and accessing such a large knowledge base may still arise. One way
of avoiding these issues was proposed by Juba~\shortcite{juba13}, building on
Khardon and Roth's {\em learning to reason} approach~\cite{kr97}: instead of
representing a knowledge base explicitly, Juba decides a query using the data
directly, and guarantees that the result is sufficient to distinguish queries
that have low validity from queries with small proofs using knowledge that
{\em could} have been learned from the data. Thus, the query is answered using
a knowledge base that is only {\em implicitly} learned. Crucially, this approach
applies to settings where some attributes are not observed in the examples used
for learning, and therefore some reasoning may be required. These attributes 
may still be mentioned in the background knowledge and query. For example, we 
may observe medical test results, but not whether a given patient
actually has a given disease.

A drawback of Juba's
approach, however, is that it provides no explicit representation of {\em what}
knowledge could have been learned to support the query. It only provides a set
of proofs for specific examples from the data set. This may not be adequately
interpretable for human oversight of the system; if possible, we would like to
inspect the knowledge that is being used to provide answers to our queries.
Moreover, there are applications for which we are more interested in what
knowledge could have been used to derive the conclusion than we are in the
conclusion itself. For example, such algorithms might be applied to learn a
screening rule for fraud detection as follows: Given a definition of 
behavior that is legitimate and a set of example transaction histories that are 
known to be legitimate, we could seek to learn what properties (if any) of the 
observed portions of these transactions can be used to guarantee that they are 
legitimate by using the definition of legitimate behavior as our query. We can 
then check whether or not these learned properties are observed to hold on future
transactions to decide whether or not they warrant further inspection. Note that 
in this case the query is presumed to hold, and the interesting part is what 
properties we can discover to justify the query. We will discuss a similar but 
more easily formalized application to learning input filters later.

In this work, we show how for a large class of {\em oblivious backward search}
algorithms for reasoning, we {\em can} explicitly identify rules that suffice
to answer queries. Thus, we explicitly identify a sufficient set of relevant 
rules from this ``implicit'' knowledge base in a goal-driven fashion. 
{\em Obliviousness} means that the only effect of the knowledge base on the 
search is to terminate branches early when the subgoal is already present in the
knowledge base. We observe that this suffices to learn such rules for logics 
such as chaining and treelike resolution where there are natural oblivious (or
nearly oblivious) algorithms, e.g., DPLL-like algorithms~\cite{dp60,dll62}.
This is achieved by deeming subgoals to be successful by adding them to the 
knowledge base when they are supported by the data; if the search is obvlivious,
then we obtain the same results as if those formulas had belonged to the 
knowledge base all along.

Our algorithms resemble algorithms arising in inductive logic programming 
(ILP)~\cite{mdr94}, in particular work by Muggleton and Buntine~\shortcite{mb88} 
that constructed rules for resolution; although Muggleton~\shortcite{muggleton91}
anticipated that a connection to PAC-Learning should exist, ILP learning theory
is quite different. The main distinction here is conceptual: ILP treats the 
input examples as {\em defining} a domain for which we seek to synthesize a 
description. By contrast, in PAC-Semantics, we are only seeking to bound the 
probability our formulas are true with respect to some probability distribution 
$D$ over valuations. Relatedly, the examples are simply drawn from $D$, 
and thus only statistically representative of it. We thus do not have complete 
knowledge of $D$ nor do we even have access to the valuation of every formula in
any given example. 
Although this bounding of probability is similar to a probability
logic (e.g., as discussed by Nilsson~\shortcite{nilsson86} for propositional 
logic or Halpern~\shortcite{Halpern90} for first-order), we stress that the
logical languages we use are simple Boolean logics such as chaining and 
resolution, and the probability bounds appear only in our semantics.

\section{Problem formulation}
We now describe the framework we use for learning and reasoning from
partial examples (interpretations). 
First, we describe learning from partial examples,
and give the key definition of concealment that captures when
a credulous strategy for learning from partial 
examples will succeed. (Skeptical learning is considered later.) We then
define a family of backward
search algorithms, the family of algorithms for which we will be able to 
introduce query-driven learning. Since our guarantees will have the form of
ensuring that these reasoning algorithms are as successful as if they had
started their search with the learned knowledge given up-front, we will need
a technical condition that the search algorithms are not too sensitive to the 
contents of the knowledge base they are given up-front---that is, ``oblivious''
to the knowledge base. Indeed, this definition will guarantee that we can 
introduce learned knowledge as the search proceeds without harming its 
performance, which is our main strategy.

\subsection{Learning and reasoning in PAC-Semantics}
Following Valiant~\shortcite{valiant00}, we
will describe our logic in terms of the linear threshold connective (a common
generalization of the usual AND and OR connectives). The linear threshold
connective has the advantage that 
%it can be learned very efficiently, e.g., by
%the WINNOW algorithm~\cite{littlestone88}, and that 
it can capture softened
versions of AND and OR.
\begin{definition}[Threshold connective]
A {\em threshold connective} for a list of $k$ formulas $\varphi_1,\ldots,\varphi_k$
is given by a list of $k+1$ integers, $c_1,\ldots,c_k, b$.
The formula $[\sum_{i=1}^kc_i\varphi_i\geq b]$ is interpreted as follows:
given a Boolean interpretation for the $k$ formulas, the connective is true if 
$\sum_{i:\varphi_i=1}c_i\geq b$.
\end{definition}
\noindent
A threshold connective expresses a $k$-ary AND connective by
taking the $c_i=1$, and $b=k$, and expresses a $k$-ary OR by taking
$c_1,\ldots,c_k, b=1$. Negation corresponds to $c_i<0$.

Although we could have taken the weights to be real-valued, on account of the
$\varphi_i$'s essentially taking values from $\{0,1\}$, any real-valued threshold
connective has an equivalent integer connective. We use integers as they are 
simpler to represent and reason about.

\begin{example}\label{threshold-ex}
 Suppose we have one, unary relation symbol $R$ and
six elements in our domain, $x_1,x_2,x_3,x_4,x_5$ and $x_6$. Now, an example of
a formula using a threshold connective is 
$$[5R(x_1)+R(x_2)+R(x_3)+R(x_4)+R(x_5)-R(x_6)\geq 4],$$ i.e., with
formulas $\psi_i=R(x_i)$, weights $c_1=5$, $c_6=-1$, and $c_2,\ldots,c_5=1$, and
a threshold of $b=4$.
\end{example}

We will assume a finite domain, and hence a finite (but possibly large) number 
of ground atomic formulas, $N$. Thus, our setting is essentially propositional.
We will formulate our exposition in terms of a simplified first-order
language without functions, in which free variables are taken to be universally 
quantified; these quantified formulas are of course equivalent to a 
quantifier-free (propositional) formula given by an AND (expressed by a 
threshold connective) over copies of the formula with each possible binding of 
the free variables. This is an example of the standard 
{\em ``propositionalization''} transformation, and a suitable family of formulas
for these purposes is described by Valiant~\shortcite{valiant00}.
Nevertheless, this representation captures standard settings of (function-free)
resolution and chaining with Horn KBs.

The main feature of PAC-Semantics is a probability distribution $D$ on
interpretations of the relation symbols, i.e., assignments of truth values to 
their groundings. Equivalently, we take each ground  atomic formula as a 
Boolean-valued random variable. We stress that we do not 
assume independence (or any other relationship) between these random variables. 
Given an interpretation drawn from $D$, the semantics of a formula are then 
defined classically. $KB\models\varphi$ denotes that $\varphi$ is a 
{\em (classically) valid formula given the knowledge base (set of formulas) 
$KB$}.
We denote the
truth value of a formula $\varphi$ under an interpretation $x$ as $\varphi|_x$. 
We may view an interpretation as a Boolean vector 
indexed by the set of ground atomic formulas.
Following Valiant~\shortcite{valiant00}, we refer to interpretations of the 
ground atomic formulas as {\em scenes}.

\begin{definition}[$(1-\epsilon)$-valid]
A sentence (i.e., formula with no free variables) $\varphi$ is said to be 
{\em $(1-\epsilon)$-valid under $D$} if the 
probability that $\varphi$ evaluates to true under an 
interpretation drawn from $D$ is at least $1-\epsilon$.
If $\epsilon=0$, we say that $\varphi$ is {\em perfectly} valid.
\end{definition}

Now, an {\em obscured scene} is a {\em partial} interpretation of the 
ground atomic formulas of the logic:

\begin{definition}[Obscured scene]
A {\em obscured scene} $\rho$ is a mapping taking ground atomic formulas to 
$\{0,1,*\}$ where * denotes an ``unknown'' value. We say that an obscured scene 
$\rho$ is {\em consistent} with an interpretation if whenever $\rho$ assigns an 
ground atomic formula a value other than $*$, the interpretation agrees with
$\rho$.
\end{definition}

We need obscured scenes because frequently our knowledge base will refer to
atomic formulas that are not observed in the data we use for learning. Sometimes
these unobserved atomic formulas take the form of properties we wish to reason
about or predict with learned rules. 
% These may also, for example, refer to the contents of variables during an 
% execution in program analysis, where the ``obscured scene'' itself might 
% simply be an example input to the program.
Following Rubin~\shortcite{rubin76} and Michael~\shortcite{michael10},
we suppose that a ``masking process'' takes interpretations drawn from 
the distribution $D$ and hides some ground atomic formulas, producing
obscured scenes. 

\begin{definition}[Masking process]\label{def-mask}
A {\em mask} is a function mapping interpretations to obscured scenes that are
consistent with the respective interpretations. A {\em masking process} $M$ is a mask-valued random 
variable (i.e., a random function). We denote the probability distribution over 
obscured scenes obtained by applying a masking process $M$ to a distribution $D$ 
over interpretations by $M(D)$.
\end{definition}

Some natural examples of masking processes that don't use the full expressive
power of the formalism are the following.
\begin{example}\label{mask-ex1}
Consider masking processes
that always produce a mask $m$ that hides the values of a subset of the ground 
atomic formulas, and never hide the rest. Such masking processes capture the
information available in a (learning-driven) program analysis application where 
the examples specify an input and nothing else. The hidden formulas would then 
encode the omitted trace of the program's execution. It also captures the 
information available in typical statistical studies in which a subset of the 
attributes of sampled members of a population are (reliably) recorded, and
the rest are omitted from the data.
\end{example}

\begin{example}\label{mask-ex2}
Consider a masking process that independently tosses a fair coin
to decide whether or not to hide the value of each ground atomic formula in a
given scene. So $M$ produces $m$ by sampling a set $S$ at random by tossing a 
fair coin for each ground atomic formula, and then the corresponding $m$ is of
the type described in Example~\ref{mask-ex1} -- it hides the ground atomic 
formulas in the random set $S$ and no others. This captures a setting where, 
due to noise corrupting a transmission, only portions of a scene can be decoded.
\end{example}

These examples do not use the ability of a masking process to optionally hide an
atom depending on the underlying truth value, but our definition allows this. 

\begin{example}\label{mask-ex3}
Consider a setting where in a survey, participants are allowed to decline to
answer a question. Naturally, one might find that when participants have (for
example) atypically high or low income, or where they possess minority political
opinions, or suffer from certain kinds of diseases, they might be less inclined
to provide an answer. Thus, in such a setting, the survey responses would be
more naturally modeled by a masking process in which, given that the sampled
member of the population falls into one of these categories, the probability of
the value being obscured is much higher than otherwise. Indeed, the model also
allows for the decision to mask to depend on more than one attribute of the 
example (note that $m(x)$ is allowed to depend on the entire interpretation $x$) 
-- e.g., $m(x)$ may omit the value of the formulas encoding the political 
inclinations only if they are relatively inconsistent with some other attributes 
of the respondent encoded by $x$.
\end{example}

For a given ground atomic formula $\alpha$, a PAC-Learning
algorithm could be used to learn a formula $\varphi$ that 
predicts $\alpha$, in which case the formula 
$[\varphi\equiv\alpha]$ is $(1-\epsilon)$-valid with probability 
$1-\delta$ over the random example scenes (for $\delta$ given to the algorithm): 
if we may take the truth value of $\alpha$ as a {\em label} and the truth values
of the ground atomic formulas as {\em attributes} for the example scenes, then 
PAC-Learning uses such examples to produce precisely such a formula $\varphi$ as
output.
%If we can learn a formula that approximately re-expresses $\alpha$ in terms of 
%the other symbols, 
Using such a rule, we could hope to infer the value of $\alpha$ in examples in 
which it is obscured. Such an approach was proposed by Valiant~\shortcite
{valiant00}, and we will return to it later.

Following Juba~\shortcite{juba13}, we will consider the following operation
that uses obscured scenes to partially evaluate quantifier-free
formulas defined using linear threshold connectives. Again, these could 
have been obtained from first-order formulas by propositionalization. Note that 
the recursive definition corresponds to a linear-time algorithm for computing 
these partially evaluated formulas:
\begin{definition}[Partial evaluation and witnessing]
Given an obscured scene $\rho$ and a quantifier-free formula $\varphi$, the 
{\em partial evaluation of $\varphi$ under $\rho$,} denoted $\varphi|_\rho$, is 
recursively defined as follows; when the partial evaluation produces a Boolean
constant, we say that the formula is {\em witnessed}:
\begin{compactitem}
\item 
A ground atomic formula $\varphi$ is replaced by its value under $\rho$ 
(i.e., it is witnessed) unless this value is *, in which case it remains $\varphi$.
\item If $\varphi=\neg\psi$ and $\psi$ is not witnessed in $\rho$, then $\varphi|_\rho=\neg(\psi|_\rho)$; otherwise, $\varphi|_\rho$ is witnessed to be $\neg(\psi|_\rho)$.
\item 
For $\varphi=[\sum_{i=1}^kc_i\psi_i\geq b]$,
\begin{compactitem}
\item $\varphi$ is witnessed true if 
$\sum_{i:\psi_i\mathrm{\ witnessed\ true}}c_i
+\sum_{i:\psi_i\mathrm{\ not\ witnessed}}\min\{0,c_i\}\geq b,$ 
\item $\varphi$ is witnessed false if 
$\sum_{i:\psi_i\mathrm{\ witnessed\ true}}c_i
+\sum_{i:\psi_i\mathrm{\ not\ witnessed}}\max\{0,c_i\}<b,$ 
\item and otherwise,
supposing that $\psi_1,\ldots,\psi_\ell$ are witnessed in $\rho$ (and 
$\psi_{\ell+1},\ldots,\psi_k$ are not witnessed),
$\varphi|_\rho$ is 
$[\sum_{i=\ell+1}^kc_i(\psi_i|_\rho)\geq d]$
where $d=b-\sum_{i:\psi_i|_\rho=1}c_i$.
\end{compactitem}
\end{compactitem}
\end{definition}

\begin{example}
Continuing Example~\ref{threshold-ex}, in any partial scene in which $R(x_1)$ is witnessed true, we find that
$\sum_{i:\psi_i\ \text{witnessed true}}c_i+\sum_{i\psi_i\ \text{not witnessed}}\min\{c_i,0\}$ is at least $5-1=4$, so the formula will be witnessed true. Likewise,
if $R(x_2),\ldots,R(x_5)$ are true and $R(x_6)$ is false, then the formula is
again witnessed true. But, if $R(x_1)$ is false and $R(x_6)$ is true, then the
formula is witnessed false, as it is if $R(x_1)$ is false and any of $R(x_2),
\ldots,R(x_5)$ are false. Finally, the formula is not witnessed if, for example,
$R(x_1)$ and $R(x_6)$ are both false, and any proper subset of $R(x_2),\ldots,
R(x_5)$ are true.
\end{example}

Following Michael~\shortcite{michael10}, we 
consider learning from example obscured scenes, provided that the value of 
$\alpha$ is not obscured in too many of the examples.
Essentially we distinguish formulas that are {\em 
perfectly} valid under the unknown distribution from those that are not even
$(1-\epsilon)$-valid for some given $\epsilon>0$. The main finding is that the 
learnability of such rules is controlled by the 
probability of observing counterexamples to flawed rules under the masking
process. 

\begin{definition}[Concealment]\label{def-conceal}
We say that a masking process $M$ is (at most) {\em $(1-\eta)$-concealing} 
with respect to a set of formulas $\calC$ and a distribution over 
interpretations $D$ if
\[
\forall\varphi\in\calC \Pr_{x\in D,m\in M}[\ \varphi\mathrm{\ witnessed\ on\ }m(x)\ |\ \varphi|_x=0\ ]\geq\eta.
\]
\end{definition}

We observe that the degree of concealment of a family of formulas depends 
on the family of formulas, the distribution over scenes, and the masking process.
\begin{example}
In the masking process of Example~\ref{mask-ex1}, in which a fixed subset of the
ground atomic formulas is never hidden and the rest are always hidden, the degree
of concealment depends on whether or not the formula in question is ever
falsified on the distribution, and if so, which ground atomic formulas in the
scene it refers to. Generally, for formulas that only refer to the observed 
ground atomic formulas, the masking process is 0-concealing (i.e., $\eta=1$,
no discounting) but for formulas that only refer to the unobserved ground
atomic formulas, the masking process is 1-concealing ($\eta=0$) and learning is,
strictly speaking, impossible. It may still be possible to infer the values of 
these attributes by reasoning, however, e.g., in the program analysis example we 
may be able to use the program code to infer the values of the program state from
an example input.
\end{example}

\begin{example}
In the case of the masking process of Example~\ref{mask-ex2}, where the ground 
atomic formulas are hidden uniformly at random, the degree of concealment may be
bounded by the number of distinct ground atomic formulas: if we observe all $k$ 
of the ground atomic formulas we certainly observe the formula being falsified, 
and this occurs with probability $1/2^k$. Thus, the masking process in this case
would be $(1-1/2^k)$-concealing ($\eta=1/2^k$). In this case, we can take the 
number of ground atomic formulas that appear in a formula as a measure of its 
complexity. Some natural fragments of logics, such as bounded-width 
resolution, limit the number of atomic formulas that may appear in the lines of 
a proof, and would thus control the degree of concealment for lines of the proof
for this masking process.
\end{example}
\noindent
In the statement of our main theorem, we include the size of the input query 
formula as a parameter, and we suppose that the degree of concealment may depend
on this parameter. This is because it is sometimes possible to bound the number
of distinct formulas that can appear in proofs by the number of proofs, which 
may be bounded similar Lemma~\ref{space-count}, below. The number of
proofs may sometimes depend, in turn, on the number and complexity of the 
premises which are encoded in the query -- consider for example, proofs with a
bounded number of lines. We have included this parameterization to facilitate 
the application of our theorem in such situations.

%%We briefly remark that concealment and our earlier notion of ``revealed''
%%formulas are clearly related; the difference is that concealment is concerned 
%%with witnessing falsehood in particular, whereas revealed just means that a
%%formula is often witnessed on the obscured scene.

Bounded concealment justifies a ``credulous'' learning strategy: rules are
satisfactory as long as we do not observe counterexamples to them. 

\begin{proposition}[Theorem 2.2 of \cite{michael10}]
\label{conceal-ce}
For any distribution $D$ over interpretations, masking process $M$, and class
$\calC$ of formulas,
if $M$ is $(1-\eta)$-concealing with respect to $\calC$ and $D$ 
and $\varphi\in\calC$ is not $(1-\epsilon)$-valid, then
$\Pr_{\rho\in M(D)}[\varphi|_\rho=0]\geq \eta\epsilon$.
Conversely, if $M$ is not $(1-\eta)$-concealing with respect to $\calC$ and 
$D$ then there exists $\varphi\in\calC$ such that
$\Pr_{\rho\in M(D)}[\varphi|_\rho=0]\leq \eta\Pr_{x\in D}[\varphi|_x=0]$
(where, note, $\varphi$ is $(1-\Pr_{x\in D}[\varphi|_x=0])$-valid).
\end{proposition}

That is, the degree of concealment controls the discounting of the probability
of observing counterexamples to formulas from $\calC$. We have actually modified
the definitions slightly from the original version by 
including the distribution in the definition of concealment, and moreover, by 
using a notion of witnessed evaluation (as opposed to the value merely being 
{\em determined} by the obscured scene) but the proof is similar.
%essentially the same.

\iffalse{
\begin{proof}
The main observation is that by definition,
\[
\Pr_{\rho\in M(D)}[\varphi|_\rho=0]
 =\Pr_{\substack{x\in D,\\m\in M}}[\varphi\mathrm{\ witnessed\ on\ }m(x)|\varphi|_x=0]
\Pr_{x\in D}[\varphi_x=0].
\]
If $\varphi$ is not $(1-\epsilon)$-valid, then 
$\Pr_{x\in D}[\varphi|_x=0]\geq\epsilon$;
if $M$ is $(1-\eta)$-concealing, then
\[
\Pr_{\substack{x\in D,\\m\in M}}[\varphi\mathrm{\ witnessed\ on\ }m(x)|\varphi|_x=0]
\geq\eta
\]
and so the first part of the claim is immediate. For the second part, suppose
that the conclusion does not hold and for all $\varphi\in\calC$
\[
\Pr_{\rho\in M(D)}[\varphi|_\rho=0]>\eta\Pr_{x\in D}[\varphi|_x=0].
\]
Then the above calculation implies that for all $\varphi\in\calC$,
\[
\Pr_{\substack{x\in D,\\m\in M}}[\varphi\mathrm{\ witnessed\ on\ }m(x)|\varphi|_x=0]>
\eta
\]
so $M$ is $(1-\eta)$-concealing with respect to $D$ and $\calC$.
\end{proof}
}\fi
In the above formulation of PAC-Learning using equivalence rules, this means
that for any incorrect hypothesis in our representation class, the value of 
$\alpha$ should be witnessed (by the masking process) on an example where the 
hypothesis is incorrect with probability at least $\eta$. It turns out that for 
many classes of interest, learning is still possible as long as the masking 
process features an $\eta$ bounded away from zero.

\iffalse{
For some applications in which we wish to screen out an adversary such as
fraud or attack detection, we desire a more ``skeptical'' learning strategy that
does not rely on the assumption of bounded concealment. In such cases, we may
use the following notion of ``testability'' from Juba~\shortcite{juba13}:
\begin{definition}[Testable]\label{def-testable}
We say that a formula $\varphi$ is {\em $(1-\epsilon)$-testable} with respect
to the distribution over partial interpretations $M(D)$ if
$\Pr_{\rho\in M(D)}[\varphi\mathrm{\ witnessed\ true\ on\ }\rho]\geq 1-\epsilon.$
\end{definition}
Note that if $\varphi$ is $(1-\epsilon)$-testable with respect to $M(D)$, it must
be $(1-\epsilon)$-valid with respect to $D$.
}\fi

In this work, we are interested in reasoning problems, in which we only
consider proofs of bounded complexity. For simplicity our formulation is again
presented in terms of ground formulas, which could have been obtained by
propositionalization. Formally, we consider the following
family of problems:

\begin{definition}[Search  problem]
Fix a logic, and let $\calP$ be a set of proofs in the logic (e.g., a fragment
of bounded complexity).
The {\em search problem for $\calP$} is then the following promise
problem: given as input a formula $\varphi$ with no free variables and a set
of formulas $KB$ such that either there is a proof of $\varphi$ in $\calP$ from
$KB$ or else $KB\not\models\varphi$, return such a proof in the former case, and
return ``Fail'' in the latter.
\end{definition}

Our first example of such a fragment is a mild generalization of the usual
forward-chaining system that was presented by Valiant~\shortcite{valiant00}. It is 
designed to utilize rules of the form $[\varphi\equiv\alpha]$ in which $\alpha$ is 
an atomic formula, i.e., of the sort obtained from PAC-Learning algorithms.
The main inference rule is {\em chaining}: given a formula of the form $[\varphi
\equiv\alpha]$ in which $\alpha$ is a ground atomic formula, and a consistent 
set of literals (atomic formulas or their negations) 
$\{\ell_1,\ldots,\ell_k\}$ such that for the obscured 
scene $\rho$ that satisfies $\ell_1,\ldots,\ell_k$ (and leaves every ground 
atomic formula not appearing in this list unassigned) $\varphi|_\rho\in\{0,1\}$, if
$\varphi|_\rho=0$, infer $\neg\alpha$, and otherwise (i.e., if $\varphi|_\rho=1$) 
infer $\alpha$. 

In chaining, it is typical to distinguish ground atomic 
formulas (``facts'') and ``rules'' of the form $[\varphi\equiv\alpha]$; in some
formulations, we suppose that only the facts appear as lines of the proof and
take the rules to be rules of inference. This limits the complexity of the
proofs that may appear in the fragment: they are just ordered lists of facts.
In particular, recall that one often considers augmenting the axioms of a logic
with a set of additional formulas -- {\em ``hypotheses''} -- that capture a
specific domain or a specific scene one wishes to reason about. Typically,
these hypotheses are the contents of the knowledge base. In the case of
forward-chaining, for example, often these hypotheses are restricted to be just 
a set of (additional) facts. We will suppose in particular that the set of 
proofs $\calP$ parameterizing our search problem may restrict the formulas that 
may be used as hypotheses in this way.

\subsection{A model of backward search algorithms}
Informally, ``backward'' (goal-directed) algorithms start with a goal query
and repeatedly generate sets of subgoal queries such that if all of the subgoal 
queries succeed -- i.e., if proofs can be found for all of these subgoal queries
-- then the algorithm can construct a proof of the goal query. Although this 
informal description suggests a recursive algorithm, it is 
highly desirable to cache the results of subgoal queries when they are 
answered---beyond the obvious time-efficiency improvements, it is often possible
for a na\"{\i}ve recursive algorithm to get stuck following a circular sequence 
of subgoals. Thus, our model of such algorithms will be stated in terms of a
graph indicating the dependency structure among the (sub)goals considered by the
algorithm. 

This graph would conventionally be an ``AND-OR'' graph:
a query would be associated with an OR node, with edges to various alternative
lists of subgoal queries, represented by AND nodes with edges to the OR nodes
corresponding to the queries in the list. The success of the algorithm
corresponds to the goal query node evaluating to `true' in the natural 
interpretation of such a graph when one more generally associates success at 
finding a proof of a (subgoal) query with a node evaluating to `true.'
Given our interest in using rules expressed using linear threshold connectives,
though, we will find it natural and convenient to replace the AND nodes with 
more general ``linear threshold'' nodes, expressing that success at the node is
achieved if an appropriate subset of the subgoals are successful. 
%(Note that this may be much more compact than listing all of the subsets explicitly.) 

\begin{definition}[Subgoal dependency graph]
A {\em subgoal dependency graph} is a (possibly infinite) directed graph $G$ in 
which the vertices are either {\em query nodes} labeled with a formula or are 
labeled with an integer {\em threshold} and have outgoing edges labeled
by integer {\em weights}. A {\em partial subgoal dependency graph} also contains,
for each threshold vertex, weights $w_+$ and $w_-$. Given sets of 
{\em successful} nodes $S$ and {\em unsuccessful} nodes $U$ in the partial
graph $G'$, each node $v$ is considered {\em successful using $S,U$} if there is 
an acyclic subgraph of $G'$ featuring $v$ as the (unique) source with sinks from 
$S\cup U$ such that $v$ has a path to every sink and at every non-query node, the sum of the weights on the 
outgoing edges to vertices in $S$ plus $w_-$ and the negative weights on outgoing edges to vertices outside $S$ or $U$ is at least the threshold. Similarly, $v$ is
{\em unsuccessful} if the sum of the weights on outgoing edges to vertices in $S$
plus $w_+$ and the positive weights of edges to vertices outside $S$ and $U$ is 
less than its threshold.
\end{definition}
\noindent
Our rule for determining when a vertex is successful or unsuccessful match
our rules for witnessing connectives true and false, respectively. The weights
$w_+$ and $w_-$ will allow us to determine witnessing with (unwitnessed) 
unrepresented vertices. They represent the total weight of unrepresented 
vertices with positive coefficients and negative coefficients, respectively;
note that the definition of witnessing uses one to establish a formula is
witnessed true and the other to establish that it is witnessed false.  Our 
sucessful vertices will intuitively represent either a provable query, or an 
applicable inference.

%(One can formally support this characterization in finite subgraphs by 
%considering a vertex failing to be successful in this sense that is most 
%distant from the source.)

The backward search algorithm is now a meta-algorithm (Algorithm~\ref
{generic-backward-alg}) parameterized by three
sub-algorithms. One algorithm, {\tt GENERATE}, generates the subgoal dependency 
graph, and another, {\tt EXPLORE}, chooses edges in the dependency graph to 
explore (as long as the algorithm is not done). The third algorithm, {\tt TEST},
generates a proof of the query if enough of the subgoal dependency graph has been
revealed so that the original goal vertex was successful (given the axioms and a 
knowledge base as successful vertices), or else indicates that the search is not 
successful yet.
Thus, the algorithm explores the subgoal dependency graph (starting from the
original query) until an appropriate collection of successful subgoals is 
discovered or the search algorithm gives up. 

\begin{algorithm}[t]
\DontPrintSemicolon
\SetKwInOut{Input}{input}\SetKwInOut{Output}{output}

\iffalse{
 that, given a finite subgoal dependency graph
$G$ with a subset of vertices marked ``unexplored'' and a set of hypotheses $H$,
chooses an unexplored vertex $v\in G$ or outputs ``FAIL.''
Algorithm {\tt GENERATE} that, given a subgoal dependency graph $G$ and a vertex
$v\in G$ either returns ``fully explored'' or returns a subgoal dependency graph
$G'$ such that $G$ is a subgraph of $G'$, obtained by eliminating one edge and 
the disconnected vertex (if one exists).
Algorithm {\tt TEST} that, given a subgoal dependency graph $G$, a query formula
$\varphi$ labeling some vertex of $G$, and a set of hypotheses $H$, either 
returns a proof of $\varphi$ from $H$, or if the vertex labeled by $\varphi$ is
not successful using $H$ and the axioms of the logic, returns ``FAIL.''}
}\fi

\Input{Query formula $\varphi$, set of formulas $KB$}

%\Output{A proof of $\varphi$ from $KB$ or {\em Fail}.}

\Begin{
\If{$\varphi\in KB$}{
 \Return{Trivial proof of $\varphi$}
}
$G\leftarrow$ vertex labeled by $\varphi$, marked ``unexplored''\\
\While{{\tt TEST}$(G,\varphi,KB)=$ FAIL}{
  \If{$v\gets${\tt EXPLORE}$(G,\varphi,KB)$ is not FAIL}{
    \If{$G'\gets${\tt GENERATE}$(G,v)$ is not ``fully explored''}{
      \If{$G'$ contains a vertex not in $G$, not labeled with $\psi\in KB$}{
        Mark the new vertex ``unexplored.''
      }
      $G\leftarrow G'$
    }
    \lElse{
      Remove ``unexplored'' mark from $v$  
    }
  }
  \lElse{
    \Return{{\em Fail}}
  }
}
\Return{{\tt TEST}$(G,\varphi,KB)$}
}

\caption{Backward search meta-algorithm}\label{generic-backward-alg}
\end{algorithm}

\begin{definition}[Backward search algorithm]
\label{bsearch-def}
%Fix a representation of subgoal dependency graphs for a class of query formulas.
A {\em backward search algorithm} is given by an instantiation 
of Algorithm~\ref{generic-backward-alg} with three algorithms:
\begin{compactitem}
\item 
An algorithm {\tt EXPLORE} that, given a finite partial subgoal dependency 
graph $G$ with a source labeled by $\varphi$ and a subset of vertices marked 
``unexplored,'' and set of formulas $KB$, chooses an unexplored vertex $v\in G$
or outputs ``FAIL.''
\item 
An algorithm {\tt GENERATE} that, given a partial subgoal dependency graph 
$G$ and a vertex $v\in G$ either returns ``fully explored'' or returns a subgoal 
dependency graph $G'$ that extends $G$ by adding one new edge starting from $v$, 
possibly to a new vertex, and reducing $w_+$ or $w_-$ by its weight if it is 
positive or negative, respectively.
%such that $G$ is a subgraph of $G'$, obtained by 
%eliminating one edge and the disconnected vertex (if one exists).
\item 
An algorithm {\tt TEST} that, given a partial subgoal dependency graph $G$,
a query formula $\varphi$ labeling some vertex of $G$, and a set of formulas 
$KB$, either returns a proof of $\varphi$ from $KB$, or if the vertex labeled by
$\varphi$ is not successful using $KB$ and the axioms of the logic, 
returns ``FAIL.''
\end{compactitem}
\end{definition}

A key property possessed by 
instantiations of the backward-search paradigm is that the graph generation 
algorithm is often {\em oblivious} to which queries are successful, the 
algorithm exploring the graph only ``terminates early'' when it encounters a 
vertex labeled by a successful query, and the algorithm recovering the proof 
depends only on the portion of the subgoal dependency graph revealed thus far 
(and the formulas appearing on query vertices appearing in it). 
We will restrict our attention to such {\em oblivious} algorithms.

\begin{definition}[Oblivious backward search algorithm]
\label{oblivious-search-def}
We say that a backward search algorithm is {\em oblivious} if:
\begin{compactenum}
\item For any partial subgoal dependency graph $G$, query $\varphi$ (contained as
a label in $G$), and set of formulas $\Phi$ not appearing as labels 
of query nodes in $G$, $\mathtt{TEST}(G,\varphi,KB)=\mathtt{TEST}(G,\varphi,KB
\cup\Phi)$.
\item For any query $\varphi$ and sets of formulas $KB$ and $\Phi$,
the execution of the algorithm on input $\varphi$ and $KB\cup\Phi$ differs from
the execution on input $\varphi$ and $KB$ only in that the sequence of vertices
proposed by {\tt EXPLORE} on input $\varphi$ and $KB\cup\Phi$ is the subsequence 
of vertices proposed on input $\varphi$ and $KB$ that omits (skips) exploring
vertices in the sequence that are successful from $KB\cup\Phi$ in the partial 
subgoal dependency graph used by the algorithm in the corresponding step.
% (on input $\varphi$ and $KB\cup\Phi$).
\end{compactenum}
\end{definition}

\subsubsection{An example: oblivious backward chaining}
We briefly note that standard backward-chaining algorithms (for Horn KBs on
ground atomic formulas) are
oblivious backward search algorithms in the sense of Definition~\ref
{oblivious-search-def}: recall that a {\em Horn clause} is a formula of the
form $[\alpha_1\wedge\cdots\wedge\alpha_k]\Rightarrow\alpha_{k+1}$ where each 
$\alpha_i$ is a ground atomic formula. $\alpha_{k+1}$ is the {\em head} whereas
$\alpha_1\wedge\cdots\wedge\alpha_k$ is the {\em body}. The knowledge base then
consists of such clauses and a set of ground atomic formulas. The query is a 
conjunction of ground atomic formulas, represented as a threshold vertex with a 
threshold equal to the number of atomic formulas in the conjunction. 

The oblivious backward chaining algorithm works as follows: 
{\tt EXPLORE} performs a depth-first search of the subgoal dependency graph, 
terminating a search early only when it encounters an atomic formula in $KB$. 
{\tt GENERATE}, on the other hand, when given a vertex corresponding to a ground
atomic formula, returns an edge to a threshold vertex corresponding to the next 
clause in $KB$ with the given atomic formula appearing in the head (in some 
fixed ordering), with a threshold equal to the number of atomic formulas in the 
body; when given a vertex corresponding to one of the clauses of $KB$ (or the 
goal conjunction), it returns an edge to the vertex labeled with the next atomic
formula in the body (also in some fixed ordering) of weight $1$. Finally, 
{\tt TEST} uses a dynamic programming algorithm to check if the query is 
successful from $KB$ and return a chaining proof if it is. The running
time is bounded by a polynomial in the size of $KB$: it runs
for a linear number of iterations, and {\tt TEST} may take quadratic time on 
each iteration. 

\begin{example}
We now illustrate the backward search algorithm for chaining of Horn rules.
Let's consider the following simple domain, concerning fragile objects. The 
relations are $fragile(x)$, $broken(x)$, $hard(x)$, $crushed(x)$, and 
$hit(x,y)$. We have a KB containing rules 
$[crushed(x)\wedge fragile(x)]\Rightarrow broken(x)$ and 
$[hit(x,y)\wedge fragile(x)\wedge hard(y)]\Rightarrow broken(x)$.
Let's suppose that the domain of objects consists of $sculpture, crate, floor,$
and $sidewalk$. As an example, we might indicate that a fragile
sculpture is crushed -- $fragile(sculpture)$ and $crush(sculpture)$ are given --
and we wish to know if $broken(sculpture)$ holds. The full subgoal dependency 
graph is depicted in Figure~\ref{bs-graph-fig}.

% example graph
\begin{figure*}[t]
\begin{center}
\includegraphics[width=.8\textwidth]{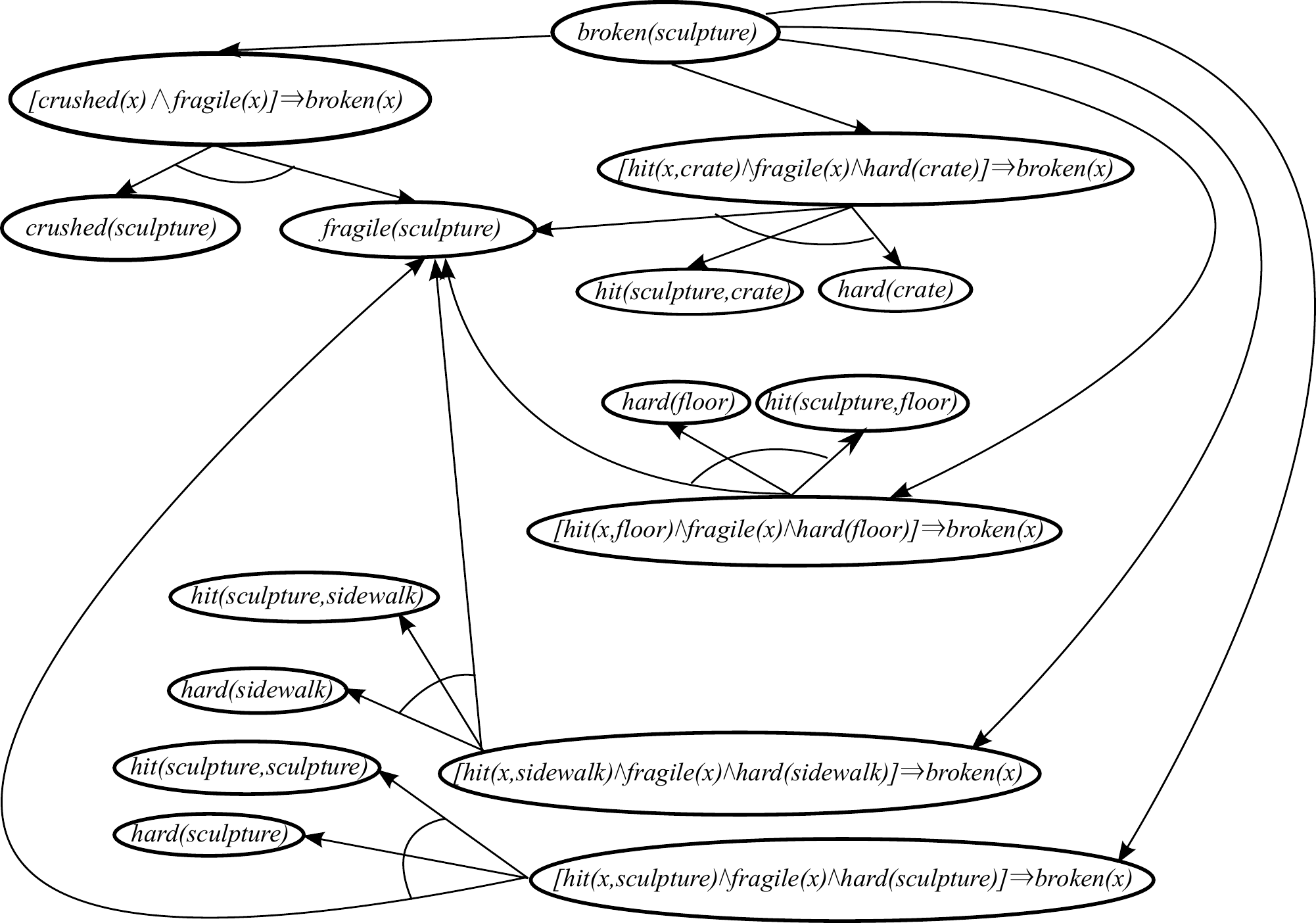}
\caption{A (AND-OR) subgoal dependency graph for the query $broken(sculpture)$ 
with backward chaining using our Horn KB. The threshold nodes are 
labeled with the corresponding groundings of the rules of the KB. ANDs 
are represented by arcs connecting edges.}\label{bs-graph-fig}
\end{center}
\end{figure*}

A short execution of the backward search algorithm starts with the query node,
$broken(sculpture)$. {\tt TEST} on $broken(sculpture)$ determines that this fact
is not given, so the algorithm invokes {\tt EXPLORE} which determines that the
vertex $broken(sculpture)$ is not yet fully explored. So, the algorithm invokes
{\tt GENERATE} on $broken(sculpture)$ which begins generating the rules that 
could produce $broken(x)$ in the head with $sculpture$ substituted for $x$. For 
simplicity, let's suppose that it first considers the rule $[crushed(x)\wedge 
fragile(x)]\Rightarrow broken(x)$, which is such a rule. This threshold vertex
represents an AND on two conjuncts, so it has a threshold of $2$.  As the graph 
is partial, it has values $w_+=2$ since both of the conjuncts (yet to be
generated) have a weight of $1>0$ and $w_-=0$ since this threshold formula does 
not use negative weights. Note that none of the nodes are yet witnessed, so 
{\tt TEST} will continue to the next iteration. 

{\tt EXPLORE} would, if it is a depth-first exploration, choose the new rule to 
explore. So it would invoke {\tt GENERATE} on the rule, which would first yield 
the subgoal node $crushed(sculpture)$ and reduce the value of $w_+$ for the
threshold by $1$ (since one of these nodes is now represented in the graph). 
Note that $crushed(sculpture)$ is in the KB, so this vertex will not be marked 
``unexplored.'' Furthermore, this node is successful, since it is in the KB. 
But, the rule that uses $crushed(sculpture)$ is not yet witnessed since it also 
needs $fragile(sculpture)$, so {\tt TEST} will continue to the next iteration. 

Now, since the new node is not marked ``unexplored,'' the depth-first 
{\tt EXPLORE} would return to the rule, and {\tt GENERATE} next generates the 
subgoal $fragile(sculpture)$, which is also in the KB, and reducing $w_+$ (at 
the rule's vertex) by $1$ to $0$. At this point, {\tt TEST} will discover that 
$crushed(sculpture)$ and $fragile(sculpture)$ suffice to witness the linear
threshold corresponding to $[crushed(sculpture)\wedge fragile(sculpture)]$, so 
this node is successful, which in turn witnesses that the original goal 
$broken(sculpture)$ is successful. Thus, {\tt TEST} returns the chaining proof 
of the query
\begin{compactenum}
\item $crushed(sculpture)$ (hypothesis)
\item $fragile(sculpture)$ (hypothesis)
\item $broken(sculpture)$ (chaining, 1 \& 2, $[crushed(x)\wedge fragile(x)]\Rightarrow broken(x)/x=sculpture$)
\end{compactenum}
Needless to say, the execution would be rather longer if the search had first
started exploring the various domain substitutions for $y$ in the rule $[hit(x,y)
\wedge fragile(x)\wedge hard(y)]\Rightarrow broken(x)$. A breadth-first search
of the graph would have been similarly longer. Note that there are no rules with
$hit(x,y)$ or $hard(x)$ in the head, so although these nodes will initially be
marked ``unexplored'' (in contrast to, say, $fragile(sculpture)$), when
{\tt GENERATE} considers these nodes it will immediately report that they are
``fully explored.'' Thus, after exploring these rules, the search will eventually
return to exploring $[crushed(x)\wedge fragile(x)]\Rightarrow broken(x)$ and
succeed as described above, even in these cases.

Now, suppose that the sculpture hits the floor, so $hit(sculpture, floor)$
holds instead of $crushed(sculpture)$. Now, we are {\em not} given $hard(floor)$
(the floor may be carpeted) so there is no proof of $broken(sculpture)$. In this
case, the backward search algorithm will generate the entire subgoal dependency
graph, before determining that the query is not provable, and terminate with
Fail.
\end{example}

\section{Query-driven learning in backward search}

\begin{algorithm}[t]
\DontPrintSemicolon
\SetKwInOut{Input}{input}\SetKwInOut{Output}{output}

\Input{Query formula $\varphi$, set of formulas $KB$, list of
obscured scenes $\rho_1\ldots,\rho_m$}

%\Output{A proof of $\varphi$ from $KB$ in $\calP$ and a set of
%formulas $\{h_1,\ldots,h_k\}$ or {\em Fail}.}

\Begin{
$G\leftarrow$ vertex labeled by $\varphi$, marked ``unexplored''\\
\While{{\tt TEST}$(G,\varphi,KB)=$ FAIL}{
  \If{$v\gets${\tt EXPLORE}$(G,\varphi,KB)$ is not FAIL}{
    \If{$G'\gets${\tt GENERATE}$(G,v)$ is not ``fully explored''}{
      \lIf{$G'$ contains a query vertex labeled by a formula $h$ not in $G$
          and for no $\rho_i$ is $h|_{\rho_i}=0$}{
         $KB\leftarrow KB\cup\{h\}$
      }
      \If{$G'$ contains a vertex not in $G$, not labeled with $\psi\in KB$}{
        Mark the new vertex ``unexplored.''
      }
      $G\leftarrow G'$
    }
    \lElse{
      Remove ``unexplored'' mark from $v$  
    }
  }
  \lElse{
    \Return{{\em Fail}}
  }
}
\Return{{\tt TEST}$(G,\varphi,KB)$}
}

\caption{Backward search with query-driven learning}
\label{mod-bw-alg}
\end{algorithm}

The relative blindness of oblivious algorithms to the contents of the knowledge 
base allows us to add new members as the search proceeds, and
obtain the same result as if we had started with them. As some algorithms may
consider families of formulas that scale with the size of the query, in our 
theorem we will parameterize our bounds on the proof size $B$ and degree of 
concealment $\eta$ by the size of the query $\ell$ (in bits). For example, 
adding clauses to a query for resolution generally increases the variety of 
clauses that may be derived. For larger 
families, we expect that the bounds grow weaker. We will use $|\varphi|$ for a
formula $\varphi$ to denote its representation size (in bits), and $|KB|$ to
denote the representation size of the $KB$, again in bits. We assume the $KB$
is represented in a way that ensures $|KB\cup H|\leq |KB|+|H|$, e.g., if it is 
represented as a string in which the elements are separated by special symbols, 
and terminated by another symbol.

\begin{theorem}%[Adding explicit query-driven learning to backward search algorithms]
\label{explicit-search-thm}
Let $\calP$ be a set of proofs such that proofs of queries of length $\ell$ 
only have proofs with $B(\ell)$-bit encodings in $\calP$ (in some fixed encoding
scheme). Suppose there is an oblivious backward search algorithm for the 
search problem for $\calP$ that on input
$\varphi$ and $KB$ over $N$ ground atomic formulas,
runs in time $T(N,|\varphi|,|KB|)$ (for
a function $T$ that is monotone increasing in $|KB|$).
Let $D$ be a distribution over scenes and $M$ be a masking process that is at 
most $(1-\eta(\ell))$-concealing for the set of formulas $\Phi$ that may be
hypotheses in proofs of formulas of length $\ell$ in $\calP$.
Then for any $\delta,\epsilon\in (0,1)$, on input $\varphi$ and $KB$ and
$\Theta(\frac{1}{\epsilon\eta(|\varphi|)}(B(|\varphi|)+\log 1/\delta))$ example 
obscured scenes, Algorithm~\ref{mod-bw-alg} using the same {\tt EXPLORE},
{\tt GENERATE}, and {\tt TEST} as the given algorithm runs in time 
$O(\frac{B(|\varphi|)}{\epsilon\eta(|\varphi|)}(B(|\varphi|)+\log 1/\delta)
T(N,|\varphi|,C))$ (for $C=|KB|+T(N,|\varphi|,|KB|)$), %\footnote{%
%For a sufficiently powerful model of computation, e.g., RAM.} 
and with probability $1-\delta$
\begin{compactitem}
\item 
returns ``Fail'' if $[KB\Rightarrow\varphi]$ is not $(1-\epsilon)$-valid with respect to $D$, or 
\item 
returns a proof of $\varphi$ from $KB\cup H'$ for a set of formulas 
$H'=\{h'_1,\ldots,h'_k\}$ such that $h'_1\wedge\cdots\wedge h'_k$ is 
$(1-\epsilon)$-valid if there exists a set of perfectly valid formulas $H$
such that there is a proof of $\varphi$ from $KB\cup H$ in $\calP$.
\end{compactitem}
\end{theorem}
We remark that the cases are {\em not} exhaustive, for two reasons. First, for 
many logics, the fragments for which proof search can be considered tractable 
are not complete. Second, it may be that the query formula is indeed
$(1-\epsilon)$-valid, but that there is not a set of formulas that we can learn 
in support of it. For example, in chaining it may be that we have $p\to r$ and
$q \to r$, and we never observe $r$ but we do observe either $p$ or $q$ to be 
true, half of the time each. So, we know $r$ is always true, but since neither 
$p$ nor $q$ is consistently true, neither one can be learned.

\begin{proof}
First suppose that there is a set of perfectly valid formulas $H$ such
that there is a proof of the query $\varphi$ from $KB\cup H$ in $\calP$. Since
we have assumed that the given backward search algorithm solves the
search problem for $\calP$, on input $\varphi$ and $KB\cup H$, the given
backward search algorithm would return a proof of $\varphi$ from $KB\cup H$.
Suppose this occurs after $t^*$ iterations.

Now, consider the sets $KB'_t$ and subgoal dependency graphs $G'_t$ used on
each respective iteration $t$ by Algorithm~\ref{mod-bw-alg}. If we consider the 
runs of the given algorithm using $KB'_t$, we see that since it is assumed to
be oblivious, on each step up to $t$, it proposes the same vertex to explore as 
Algorithm~\ref{mod-bw-alg} (i.e., it only omits the vertices that are successful
from $KB'_t$) and thus generates the subgoal dependency graph $G'_t$ on the 
$t$th step. Furthermore, we note that since $KB'_t$ contains the subset of $KB
\cup H$ that appears as labels in $G'_t$, every vertex of $G'_t$ that is not 
deemed successful from $KB'_t$ is likewise not successful from $KB\cup H$ in 
the graph generated on the $t$th step on input $\varphi$ and $KB\cup H$. 
Therefore, if Algorithm~\ref{mod-bw-alg} runs for $t^*$ iterations, $\varphi$ is
successful from $KB'_{t^*}$. Moreover, as {\tt EXPLORE} continues to propose the
vertices of $G'_t$ not successful from $KB'_t$ on each step (since it is 
oblivious) until the vertex labeled by $\varphi$ is successful, another 
unsuccessful vertex must exist. Therefore the algorithm can only terminate 
before $t^*$ iterations if {\tt TEST} returns a proof and so either way 
Algorithm~\ref{mod-bw-alg} returns a proof of $\varphi$ in $\calP$ from some 
$KB'_t$.

The running time is bounded as follows: we see that the algorithm runs for no 
more iterations than before, and the final size of $KB'_t$ is at most 
$|KB|$ plus the algorithm's running time, since the algorithm creates vertices
representing each formula added to $KB'_t$. Ignoring the time to test the
proposed vertices for membership in $KB'_t$, the running time may be at most 
$T(N,|\varphi|,|KB'_t|)\leq T(N,|\varphi|,|KB|+T(N,|\varphi|,|KB|))$. Now, each
formula tested appears as a premise in some proof in $\calP$ and therefore has
a representation of size at most $B(|\varphi|)$. Since we can determine the
witnessed value of a formula on a given obscured scene in linear time 
the time bound follows.

We now argue that with probability $1-\delta$ over the example obscured scenes
provided to the algorithm as input, any proof that could be returned by 
Algorithm~\ref{mod-bw-alg} uses a knowledge base $KB\cup H'$ such that
the formula $h'_1\wedge\cdots\wedge h'_k$ (for $H'=\{h'_1,\ldots,h'_k\}$) is 
$(1-\epsilon)$-valid. This will establish the theorem as the existence of such a 
proof guarantees that the query $\varphi$ is $(1-\epsilon)$-valid, so if the first
case holds, the algorithm cannot produce a proof except with probability 
$\delta$; and likewise, in the second case, the proof returned by the algorithm 
is satisfactory with probability $1-\delta$.

To this end, we note that since $M$ is assumed to be $(1-
\eta(|\varphi|))$-concealing with respect to the premises that may appear on any
proof of $\varphi$ in $\calP$, for any proof using a set of premises that is 
{\em not} $(1-\epsilon)$-valid, Proposition~\ref{conceal-ce} shows that each 
example produces an obscured scene $\rho$ for which $h'_i|_\rho=0$ for some 
$h'_i$ used as a premise with probability at least $\epsilon\eta(|\varphi|)$. 
Therefore, in a sample of $\Omega(\frac{1}{\epsilon\eta(|\varphi|)}(B(|\varphi|)+
\log 1/\delta))$ independent obscured scenes, the probability that no premise of
such a proof has $h'_i|_\rho=0$ for any $\rho$ in the sample is at most
$\delta\cdot 2^{-B(|\varphi|)}$; as the proofs in $\calP$ for $\varphi$ have
encodings of at most $B(|\varphi|)$ bits, a union bound over these proofs gives
that the overall probability of some proof having no $\rho$ in the sample for
which $h'_i|_\rho=0$ for some premise in the proof is at most $\delta$.
Now, as the algorithm only returns proofs using premises contained in the sets 
$KB'_t$ which in particular do not contain formulas $h'$ such that 
$h'|_\rho=0$ for any $\rho$ in the sample, we see that with probability $1-
\delta$, the algorithm does not return a proof with a set of premises that is 
not $(1-\epsilon)$-valid, as needed.
\end{proof}

\subsubsection{Example: backward chaining}
If we represent chaining proofs by the sequence of inferred ground atomic 
formulas (using $\log N$ bits for each), then since any proof needs only write
down an atomic formula at most once, we can use the bound $B=N\log N$ in the 
statement of Theorem~\ref{explicit-search-thm}. (The size of the query $\ell$ is 
always the length of a single ground atomic formula, $\log N$ bits.)
% ($\ell$ is only interesting in a richer proof system like resolution, which we 
% will discuss later.)
So, after applying the transformation depicted in Algorithm~\ref{mod-bw-alg} to
the standard backward-chaining algorithm, Theorem~\ref
{explicit-search-thm} establishes that this modified backward chaining algorithm
finds chaining proofs of queries using not only ground atomic formulas from the 
explicitly given $KB$, but also additional formulas that are almost always true.
The modified algorithm automatically supplements a given $KB$ with any such
additional ground atomic formulas that suffice to complete some chaining proof
of a query if one exists, provided further that the masking process has bounded
concealment with respect to ground atomic formulas---meaning here, the masking 
process leaves the value of each ground atomic formula present with some bounded
probability when it is false. 

% KB is given as a collection of sequents; chaining is again the only rule

% Backward chaining is a query-driven proof search algorithm =>
%  thm-(...) applies

\iffalse{
We also note that a similar backward chaining algorithm could be used with
the kind of formulas explicitly learned in Valiant's model~\cite
{valiant00}---rules of the form $\phi\equiv\tau$ for a literal $\tau$ and $\phi$ 
of the form of a linear threshold function over literals. These explicitly 
learned rules could augment an explicitly designed KB.
}\fi

Chaining does not feature a rule of inference that allows new rules to be
derived, so chaining algorithms never need to consider rules outside $KB$. 
Hence, the proposed generic transformation does not learn such rules.
Other logics such as resolution, which allow richer kinds of formulas to be 
derived, yield more interesting query-driven learning.
In principle one could also consider a variant of backward-chaining in which
rules are learned as well. The difficulty with such a variant lies in
controlling the complexity of the search.

\subsection{Skeptical query-driven learning}
Theorem~\ref{explicit-search-thm} relies on an assumption of bounded concealment,
and uses a credulous learning strategy of searching for hypotheses for which we 
do not possess counterexamples. A more conservative strategy would replace the 
condition ``for no $\rho_i$ is $h|_{\rho_i}=0$'' in Algorithm~\ref{mod-bw-alg} 
with the condition ``for all $\rho_i$, $h|_{\rho_i}=1$.'' An $h$ that is 
witnessed true with probability 1
%$1$-testable 
will pass this condition. Moreover, if $h|_{\rho_i}=1$ for $\rho_i=
m_i(x_i)$ (where $m_i$ is drawn from $M$ and $x_i$ is drawn from $D$), then 
$h|_{x_i}=1$ as well. Thus, the probability that an $h$ that is {\em not}
$(1-\epsilon)$-valid with respect to $D$ passes this test for $m$ examples is
less than $(1-\epsilon)^m$. We can use this observation in place of Proposition 
\ref{conceal-ce} to finish an analogous proof, given that we are searching for 
an
%perfectly {\em testable} 
$H$ 
that is always {\em witnessed}
rather than one that is merely perfectly 
{\em valid}. We leave further details to an interested reader.

\subsubsection{Example application: learning input filters}
We note that our transformation for skeptical learning of  rules could be applied to
static program analysis algorithms to automate the generation of sound and 
approximately complete input filters. For example, the SIFT system~\cite{lskr14} 
is based on a set of sound transformation rules for analyzing integer overflow 
errors in a given program. Its associated static analysis algorithm uses the 
knowledge base given by these transformation rules and the program code to 
generate symbolic expressions that are propagated backwards from integer 
operations that might produce overflows, until they refer only to program inputs.
The condition expressed by these expressions may then be used to filter out 
inputs that do not satisfy the condition. The soundness of SIFT's transformation 
rules ensures that no input that passes this condition generates an integer 
overflow error.

We can interpret SIFT as taking the safety property of ``no integer overflows 
occur at the given point in the program'' as a query, and seeking a proof of this
query using a property of the input, together with the transformation rules and
program code. Note that once SIFT generates expressions that refer only to input 
values, the conditions they express are witnessed against example inputs, if they 
are $(1-\epsilon)$-valid for inputs in practice. Thus, we can apply the transformation of the
skeptical variant of Algorithm~\ref{mod-bw-alg} to the static analysis 
algorithm used by SIFT, and we would obtain an algorithm with a similar 
termination condition, with the added requirement that the condition found must
be witnessed on a set of given examples.

Thus, the distinction between the approach taken by Long et al.~and our 
transformation is that SIFT has no guarantee that it produces a rule that is 
testable on real inputs.\footnote{%
Nevertheless, it has been empirically demonstrated that for a certain piece of 
software, SIFT produces rules that are satisfied by real inputs with high
probability on a natural distribution~\cite{jmlsr15}.} SIFT does not use any training data, and simply 
terminates once it finds a rule that refers only to the input; thus, while this 
rule is sound -- inputs that satisfy it provably do not generate integer overflow
errors -- it has no guarantee of completeness, approximate or otherwise. Indeed, 
SIFT conservatively considers all possible execution paths and relatedly uses 
some hard-coded limits on, for example the number of loop iterations it considers
in search of loop invariants, to ensure termination. If this limit is exceeded 
because the loop potentially relies on an unbounded number of values (for 
example), then SIFT simply fails to find a condition. Thus, there is scope for a 
backtracking variant of SIFT's static analysis algorithm to obtain greater 
completeness by iteratively refining a symbolic condition. Under our 
transformation, we would then obtain an algorithm that searches for a condition 
on the inputs that empirically satisfies witnessing for a large fraction of a 
training set of benign inputs.

\subsection{Query-driven learning in treelike resolution}
\label{explicit-discovery}

Recall that {\em resolution} is a logic that operates on {\em clauses} 
(ORs of literals), using two kinds of inference rules: {\em cut} and (optionally)
{\em weakening}.
Cut takes two clauses $C=C'\vee \alpha$ and $D=D'\vee \neg \alpha$ 
and produces a clause of the form $C'\vee D'$.
(More general variants use substitutions to unify distinct atomic formuals
$\alpha$ and $\neg\alpha'$.) Weakening, by contrast, simply adds new literals
to the clause. Resolution is typically used to prove a DNF (an OR of ANDs) by 
deriving an (unsatisfiable) empty clause from its negation, a CNF.

DPLL~\cite{dp60,dll62} is another example of such a goal-directed search 
algorithm. In particular, bounded variants of DPLL that (efficiently) solve the 
proof search problem for space-bounded treelike resolution are known~\cite
{kullmann99,et01}. This is the fragment of resolution refutations that can be
derived while {\em (i)} storing at most $s$ clauses in memory simultaneously and
{\em (ii)} ``forgetting'' a clause as soon as it is used in a proof step (so 
that it must be derived again if it is needed again). 
%Other characterizations
%are known, e.g., as proofs corresponding to trees of bounded
%``Strahler number''~\cite{ablm08}.
This is a second example of an algorithm into which we 
can introduce query-driven explicit learning along the lines of Theorem~\ref
{explicit-search-thm}. The algorithm will be more interesting because it will
discover a CNF that suffices to complete the proof out of an exponentially
large (in terms of the number of ground atomic formulas) set of possible such 
formulas. 

Unfortunately, we cannot apply Theorem~\ref{explicit-search-thm} directly, as
the standard algorithm is not actually {\em oblivious} in our strict sense. The 
difficulty is that the base case of the recursive algorithm involves a search 
for a hypothesis clause for which we can derive the current clause via weakening.
This ``looking ahead'' into the hypothesis set means that the algorithm may not 
engage in exactly the same search pattern if we modify the hypothesis set during 
the search. Nevertheless, the use is innocuous enough that essentially the same 
technique can be used to give a variant of the algorithm that learns an explicit 
set of clauses from the data.

\begin{algorithm}[t]
\DontPrintSemicolon
\SetKwFunction{LSS}{Learn+SearchSpace}
\SetKwInOut{Input}{input}\SetKwInOut{Output}{output}

\Input{CNF $\varphi$, integer space bound $s\geq 1$, current clause $C$,
list of obscured scenes $\rho^{(1)},\ldots,\rho^{(m)}$.}
%\Output{A space-$s$ treelike resolution proof of $C$ from clauses in $\varphi$
%and a CNF $H$ consistent with the obscured scenes,
%or ``none'' if no such proof exists.}

$\LSS(\varphi,s,C,(\rho^{(1)},\ldots,\rho^{(m)}))$
\Begin{
\If{No $\rho^{(i)}$ has $C|_{\rho^{(i)}}=0$}{
\Return{A proof asserting $C$ (from $H$).}}
\ElseIf{$C$ is a superset of some clause $C'$ of $\varphi$}{
\Return{The weakening derivation of $C$ from $C'$.}}
\ElseIf{$s>1$}{
\ForEach{Literal $\ell$ such that neither $\ell$ nor $\bar{\ell}$ is in $C$}{
\If{$\Pi_1\leftarrow$\LSS$(\varphi,s-1,C\vee \ell,(\rho^{(i)}:\ell|_{\rho^{(i)}}
\neq 1))$
does not return {\em none}}{
\If{$\Pi_2\leftarrow$\LSS$(\varphi,s,C\vee\bar{\ell},(\rho^{(i)}:
\ell|_{\rho^{(i)}}\neq 0))$ does not return {\em none}}{
\Return{Derivation of $C$ from $\Pi_1$ and $\Pi_2$}
}
\lElse{
\Return{{\em none}}
}
}
}
}
\Return{{\em none}}
}
\caption{Space-bounded resolution with learning}\label{mod-search-space}
\end{algorithm}

Our analysis of explicit query-driven learning (following 
Theorem~\ref{explicit-search-thm}) still requires a bound on the lengths of the 
proofs in terms of the space (and number of ground atomic formulas). 
We will need the observation that DPLL-like algorithms produce {\em
normal} resolution proofs:

\begin{definition}[Normal]
We will say that a resolution proof is {\em normal} if in its corresponding DAG:
\begin{compactenum}
\item All outgoing edges from Cut nodes are directed to Cut nodes.
\item The clauses labeling any path to the sink from a Cut node contain literals
using every variable along the path.
\item A given variable is used in at most one cut step and at most one weakening
step along every path from a source to a Cut node.
\end{compactenum}
\end{definition}

\noindent
Our bound is now given in the following lemma, slightly modified from the work of 
Ehrenfeucht and Haussler~\shortcite{eh89}. 

%For completeness, we also include the proof.

\begin{lemma}[Lemma 1, Ehrenfeucht and Haussler 1989]\ 
\label{space-count}
Let $k$ be the number of nodes in a space-$s$ normal treelike resolution 
proof over $N$ ground atomic formulas where $N\geq s\geq 1$. Then,
\begin{compactenum}
\item $2^s-1\leq k\leq 2(eN/(s-1))^{s-1}$
where $e$ is the base of the natural logarithm.
\item There are at most $(8N)^{(eN/(s-1))^{s-1}}$ space-$s$ normal
treelike proofs that do not use weakening.
\end{compactenum}
\end{lemma}

\begin{theorem}%[Explicit query-driven learning for space-bounded treelike resolution]
\label{space-explicit-thm}
Let a clause $C$ and a (KB) CNF $\varphi$ be given. Suppose the examples are
drawn from a masking process that is $(1-\eta)$-concealing
with respect to CNFs of size $(eN/s-1)^{s-1}$ for the distribution $D$;
suppose further that $\varphi$ is perfectly valid with respect to $D$ and there 
exists some other perfectly valid CNF $H$ for which there is a space-$s$ 
treelike resolution proof of $C$ from $\varphi\wedge H$. Then, Algorithm~\ref
{mod-search-space} run on $\varphi$ and $C$ with parameter $s$ on a sample of size
$\Theta((N^{s-1}\log N+\log\frac{1}{\delta})\frac{1}{\eta\epsilon})$ runs in time
$O(\frac{N^{2(s-1)}|\varphi|}{\eta\epsilon}(N^{s-1}\log N+\log\frac{1}{\delta}))$
and returns a proof of
$C$ from $\varphi\wedge H'$ for some CNF $H'$ of size $O(N^{s-1})$ that is
$(1-\epsilon)$-valid with respect to $D$ with probability $1-\delta$.
Similarly, if $[\varphi\Rightarrow C]$ is not $(1-\epsilon)$-valid, Algorithm~\ref
{mod-search-space} rejects in the same time bound using the same number of 
examples.
\end{theorem}
\begin{proof}
First, consider the (possibly exponential size) formula $\tilde{H}$ 
consisting of all clauses that are consistent with the obscured scenes 
$\rho^{(1)},\ldots,\rho^{(m)}$. Noting we don't need to use weakening for 
clauses from $\tilde{H}$ since any clause that would be derived by weakening 
is also in $\tilde{H}$, the standard analysis of algorithms for space-bounded
treelike resolution~\cite{kullmann99} establish that  Algorithm~\ref
{mod-search-space} finds a normal space-$s$ treelike resolution proof of the 
input $C$ from $\varphi\wedge\tilde{H}$ and runs in time 
$O(m|\varphi|N^{2(s-1)})$ on a sample of size $m$. When a $1$-valid space-$s$ 
treelike proof is assumed to exist, it is in particular consistent with every
sample with probability $1$, so the algorithm outputs some proof in this case. 
It remains only to show that, for a 
sample of size $\Theta(\frac{1}{\eta\epsilon}(N^{(s-1)}\log N+\log\frac{1}
{\delta}))$, any such proof has its leaves labeled with a $(1-\epsilon)$-valid CNF
with probability at least $1-\delta$ (and so in particular, the algorithm cannot
return a proof if $[\varphi\Rightarrow C]$ is not $(1-\epsilon)$-valid).

Consider any space-$s$ normal treelike resolution proof that has 
leaves that are {\em not} labeled by a $(1-\epsilon)$-valid CNF. Since, by part 
1 of Lemma~\ref{space-count}, this CNF has at most $O(N^{s-1})$ clauses, 
Proposition~\ref{conceal-ce} shows that each example produces a counterexample 
to this CNF with probability at least $\eta\epsilon$. Thus, in a sample of size 
$\Omega(\frac{1}{\eta\epsilon}(N^{s-1}\log N+\log\frac{1}{\delta}))$, the 
probability that this CNF is consistent with the sample is at most $\delta N^{-
\Omega(N^{s-1})}$. Now, by part 2 of Lemma~\ref{space-count}, there are at 
most $N^{O(N^{s-1})}$ possible proofs (up to weakening steps), where we notice
that Algorithm~\ref{mod-search-space} introduces weakening steps iff the clause 
is consistent with some clause of the (fixed) input CNF $\varphi$. Therefore, the
algorithm indeed considers at most $N^{O(N^{s-1})}$ distinct proofs, so for 
a suitable choice of constant in the sample size, a union bound gives the 
probability that Algorithm~\ref{mod-search-space} encounters some proof from a 
CNF that is {\em not} $(1-\epsilon)$-valid but consistent with the sample is at 
most $\delta$. Since the algorithm only outputs proofs that are consistent
with the sample, we therefore find that any proof it outputs is derived from a 
CNF that is $(1-\epsilon)$-valid with respect to $D$ with probability at least
$1-\delta$, establishing the claim.
\end{proof}

Lemma~\ref{space-count} also shows treelike proofs of size $k$ are necessarily 
space-$\log(k+1)$, so this gives a quasipolynomial time and sample complexity
algorithm for general treelike proofs.

\section{Directions for future work}

% Learning with robustness to counterexamples?
One main direction for future work concerns a slightly relaxed variant of the
skeptical learning strategy, in which we try to learn rules that are 
simultaneously witnessed true with maximum frequency, which may be less than 
$1$.
What makes this formulation challenging is that it cannot be achieved by just
seeking that the formulas at the individual nodes are each witnessed with some
probability $1-\epsilon$: it matters which $\epsilon$-fraction of example scenes
are not witnessed across the various rules,
since we are interested in how many examples in total fail to witness at least
one of the hypotheses used in the proof. This formulation seems to be a much
harder computational problem, so it makes sense to ask what kind of approximate
solutions can be obtained.

% Natural questions for learning of expressions with quantifiers --
% is there a nice model that allows this? Are classical interpretations too
% restrictive?
Another direction for work concerns relational reasoning algorithms. The usual
algorithms for backward search in relational reasoning generate substitutions
during the search, rather than generating grounded versions of the rules as we
do here. Although this creation of substitutions may create a ``non-oblivious''
search (as viewed on the ground atomic formulas), we suspect that it may again
be innocuous enough that our main theorem will still hold, as in the example of
weakening in resolution. Such algorithms would be much more efficient, of 
course.
The main difference is that we now need to identify unifiers against the set of
(all) example scenes (in addition to formulas in the KB).

A related question, along the lines of the above but perhaps more ambitious is,
can we learn universally quantified expressions in infinite or open domains? 
Again, our current method generates ground expressions, and so we can only
consider domains with a reasonable number of elements. But, the credulous 
bounded
concealment definition at least raises the possibility that we might be able
to infer a universally quantified statement based merely on the lack of observed
counterexamples. Of course, such inferences lean heavily on the bounded 
concealment assumption.

A final direction is the development of further applications of such
algorithms. In addition to the natural application of the
algorithm in extracting knowledge that is suitable for human inspection in query
driven learning, we have identified two domains in which the kind
of sound and approximately complete rules our method generates would provide a 
useful filtering criterion. It is natural to ask if there are any others.

\section*{Acknowledgements}
This work was supported in part by ONR grant number N000141210358 while the
author was affiliated with Harvard University, and by NSF Award CCF-1718380. 
This manuscript was prepared in part while the author was visiting the Simons 
Institute for the Theory of Computing and Google.

\bibliographystyle{named}
\bibliography{robust}

\end{document}